\documentclass{article}

 \usepackage[preprint]{neurips_2026}

\usepackage{amsmath,amssymb,amsthm}
\usepackage{algorithm}
\usepackage{algpseudocode}
\usepackage{graphicx}
\usepackage{booktabs}
\usepackage{hyperref}
\usepackage{tikz}
\usepackage{enumitem}
\usepackage{pgfplots}
\pgfplotsset{compat=1.18}
\usetikzlibrary{patterns,fillbetween,arrows.meta,positioning,fit,backgrounds,calc,shapes.geometric}
\usepackage{colortbl}

\newtheorem{theorem}{Theorem}[section]
\newtheorem{lemma}[theorem]{Lemma}
\newtheorem{proposition}[theorem]{Proposition}

\newtheorem{claim}[theorem]{Claim}
\theoremstyle{definition}
\newtheorem{definition}[theorem]{Definition}

\theoremstyle{remark}

\newtheorem*{remark*}{Remark}

\newcommand{\E}{\mathbb{E}}
\renewcommand{\Pr}{\mathop{\mathrm{Pr}}}
\newcommand{\R}{\mathbb{R}}

\newcommand{\OPT}{\mathrm{OPT}}
\newcommand{\LP}{\mathrm{LP}}
\newcommand{\ALG}{\mathrm{ALG}}
\newcommand{\obj}{\mathrm{obj}}

\DeclareMathOperator{\VC}{VC}

\title{\textbf{On the Sparsifiability of Correlation Clustering:
Approximation Guarantees under Edge Sampling}}

\author{%
  Ibne Farabi Shihab\thanks{Equal contribution.} \\
  Department of Computer Science \\
  Iowa State University \\
  \texttt{ishihab@iastate.edu} \\
  \And
  Sanjeda Akter\footnotemark[1] \\
  Department of Computer Science \\
  Iowa State University \\
  \texttt{sanjeda@iastate.edu} \\
  \And
  Anuj Sharma \\
  Department of Civil, Construction and Environmental Engineering \\
  Iowa State University \\
  \texttt{anujs@iastate.edu} \\
}

\date{}

\begin{document}
\maketitle

\begin{abstract}
Correlation Clustering (CC) is a fundamental unsupervised learning primitive whose strongest LP-based approximation guarantees require $\Theta(n^3)$ triangle inequality constraints and are prohibitive at scale.
We initiate the study of \emph{sparsification--approximation trade-offs} for CC, asking how much edge information is needed to retain LP-based guarantees.
We establish a structural dichotomy between pseudometric and general weighted instances.
On the positive side, we prove that the VC dimension of the clustering disagreement class is exactly $n{-}1$, yielding additive $\varepsilon$-coresets of optimal size $\tilde{O}(n/\varepsilon^2)$; that at most $\binom{n}{2}$ triangle inequalities are active at any LP vertex, enabling an exact cutting-plane solver; and that a sparsified variant of LP-PIVOT, which imputes missing LP marginals via triangle inequalities, achieves a robust $\frac{10}{3}$-approximation (up to an additive term controlled by an empirically computable imputation-quality statistic $\overline{\Gamma}_w$) once $\tilde{\Theta}(n^{3/2})$ edges are observed, a threshold we prove is sharp.
On the negative side, we show via Yao's minimax principle that without pseudometric structure, any algorithm observing $o(n)$ uniformly random edges incurs an unbounded approximation ratio, demonstrating that the pseudometric condition governs not only tractability but also the robustness of CC to incomplete information.
\end{abstract}

\section{Introduction}

Modern data pipelines routinely produce pairwise similarity judgments at scales where the number of pairs reaches billions.
Correlation Clustering (CC) \cite{BansalBlumChawla2004} is the canonical formalization: given a signed complete graph $G = (V, E^+ \uplus E^-)$ encoding pairwise agreements ($+$) and disagreements ($-$), partition the vertices to minimize the total weight of \emph{misclassified} edges (positive edges cut across clusters, negative edges left within).
CC underpins community detection \cite{ChenSanghaviXu2012}, entity resolution \cite{Kalashnikov2008}, and label aggregation \cite{Agrawal2009}, and its theoretical study has driven a rich line of approximation algorithms.

The strongest known guarantees \cite{CaoCohenAddad2024, CohenAddadLeeNewman2023} solve a linear programming (LP) relaxation with $\binom{n}{2}$ variables and $3\binom{n}{3}$ triangle inequality constraints, then round the fractional solution.
This LP-based pipeline achieves approximation ratios approaching $1.485$ \cite{CohenAddadEtAl2024stoc}, but the $\Theta(n^3)$ constraint count makes it prohibitive for large $n$.
Practitioners therefore resort to combinatorial heuristics such as PIVOT \cite{AilonCharikarNewman2008}, KwikCluster, and distributed methods \cite{BehnezhadCharikar2022}, which scale well but forfeit the LP's theoretical edge.

A natural middle ground is \emph{sparsification}: retain only a subset of edges or constraints, solve a smaller problem, and hope the approximation guarantee degrades gracefully.
Sparsification is well understood for cuts \cite{BenczurKarger2015}, spectra \cite{SpielmanTeng2011}, and clustering objectives like $k$-means \cite{FeldmanLangberg2011}, but for correlation clustering the picture is surprisingly blank.
Streaming \cite{Fichtenberger2021, BehnezhadCharikarMaTan2023}, distributed \cite{CohenAddadLattanzi2021, GuoMitrovicVassilvitskii2021}, sublinear \cite{AssadiWang2022, CaoCohenAddad2024}, sketching \cite{GraphSketches2025}, and query-efficient \cite{MiyauchiEtAl2024} approaches address scalability from various angles, yet none formally characterizes the fundamental trade-off: \emph{how many edges must an algorithm observe to preserve LP-based approximation guarantees, and when does sparsification provably fail?}

We answer this question by establishing a \textbf{structural dichotomy}.
When edge weights satisfy the triangle inequality (the \emph{pseudometric} condition), the metric structure propagates information across unobserved edges, and aggressive sparsification is possible.
Without this structure, each edge carries independent information, and even modest sparsification destroys approximability.

\subsection{Contributions}

\begin{enumerate}[leftmargin=2em]
\item \textbf{Additive edge coreset via VC dimension (Theorem~\ref{thm:coreset}).}
For any weighted CC instance on $n$ vertices with total weight $W$, we prove $\VC(\mathcal{H}_G) = n - 1$ exactly and use this to construct \emph{additive $\varepsilon$-coresets} of size $\tilde{O}(n / \varepsilon^2)$: the cost of every clustering is preserved within additive error $\varepsilon W$.

\item \textbf{Relative coreset above a threshold (Theorem~\ref{thm:sensitivity-coreset}).}
By reducing to the additive coreset, we show that weight-proportional sampling also yields a \emph{relative $\varepsilon$-coreset} for all clusterings whose cost exceeds a user-chosen threshold $\tau$, with sample complexity $\tilde{O}(d \cdot W^2 / (\varepsilon^2 \tau^2))$.

\item \textbf{Constraint sparsification via complementary slackness (Theorem~\ref{thm:constraint-sparsify}).}
At an optimal LP vertex solution, at most $\binom{n}{2}$ triangle inequality constraints are linearly independent among the active set.
We give a correct cutting-plane algorithm that identifies the active constraints (Lemma~\ref{lem:cutting-plane}).

\item \textbf{LP-PIVOT on sparse instances (Theorems~\ref{thm:sparse-pivot} and~\ref{thm:robust-sparse-pivot}).}
We design \textsc{Sparse-LP-Pivot}, a variant of LP-PIVOT using triangle-inequality imputation for missing edges.
We identify a sharp witness-formation threshold: $\tilde{\Theta}(n^{3/2})$ samples are necessary and sufficient for triangle-based imputation to have witnesses (Lemma~\ref{lem:witness-density}).
Under a good-witness condition, we recover a $(\frac{10}{3}+\varepsilon)$-approximation; without it, we give a robust bound parameterized by an empirically measurable imputation-quality statistic $\overline{\Gamma}_w$ (Definition~\ref{def:gamma-bar}).
We also characterize witness quality for tree-metric LP solutions (Proposition~\ref{prop:tree-metric}).

\item \textbf{Lower bound for general weights (Theorem~\ref{thm:lower-bound}).}
Without the pseudometric assumption, any algorithm making $o(n)$ edge queries cannot achieve a constant-factor approximation.
We prove this via Yao's minimax principle on a ``hidden heavy negative clique'' construction.
\end{enumerate}

Taken together, these results reveal that the pseudometric condition, already known to be necessary for constant-factor approximability in weighted CC \cite{CharikarGao2024}, is in fact a \emph{stability condition}: it controls how gracefully LP-based algorithms degrade under incomplete information.

\subsection{Related Work}

CC was introduced by \cite{BansalBlumChawla2004} with a constant-factor approximation.
Subsequent LP-based work achieved ratios of $3$ \cite{AilonCharikarNewman2008}, $2.06$ \cite{ChawlaMakarychev2015}, and $1.485 + \varepsilon$ \cite{CohenAddadEtAl2024stoc, CaoCohenAddad2024}.
Pseudometric-weighted CC was studied by \cite{CharikarGao2024}; chromatic CC by \cite{Bonchi2012, XiuHan2022, FanLeeLee2025}; and dynamic CC by \cite{CohenAddadLattanziMaggiori2024}.
The scalability approaches cited in the introduction \cite{Fichtenberger2021, BehnezhadCharikarMaTan2023, GuoMitrovicVassilvitskii2021, CohenAddadLattanzi2021, AssadiWang2022, GraphSketches2025, MiyauchiEtAl2024} address computational efficiency but do not characterize the information-theoretic cost of sparsification.
On the graph sparsification side, spectral sparsifiers \cite{SpielmanTeng2011}, cut sparsifiers \cite{BenczurKarger2015}, and coresets for geometric clustering \cite{FeldmanLangberg2011} are well developed, but to our knowledge no prior work provides formal edge-sparsification guarantees for correlation clustering.

\section{Preliminaries}\label{sec:prelim}

We begin by establishing notation and recalling the LP relaxation for correlation clustering.
Throughout, we use $\tilde{O}(\cdot)$ and $\tilde{\Omega}(\cdot)$ to suppress polylogarithmic factors in $n$, $1/\varepsilon$, and $1/\delta$.

Let $G = (V, E^+ \uplus E^-)$ be a complete signed graph on $n = |V|$ vertices with edge weights $w : \binom{V}{2} \to \R_{\geq 0}$.
A \emph{clustering} $\mathcal{C} = \{C_1, \ldots, C_k\}$ is a partition of $V$.
For a clustering $\mathcal{C}$, define the indicator $\sigma_{\mathcal{C}}(u,v) = 0$ if $u,v$ are in the same cluster and $\sigma_{\mathcal{C}}(u,v) = 1$ otherwise.
The \emph{cost} of $\mathcal{C}$ is:
\begin{equation}\label{eq:cc-cost}
\obj(\mathcal{C}) = \sum_{uv \in E^+} w_{uv} \cdot \sigma_{\mathcal{C}}(u,v) + \sum_{uv \in E^-} w_{uv} \cdot (1 - \sigma_{\mathcal{C}}(u,v)).
\end{equation}

\begin{definition}[Pseudometric weights]\label{def:pseudometric}
Edge weights $w$ form a \emph{pseudometric} if $w_{uv} + w_{vw} \geq w_{uw}$ for all $u,v,w \in V$.
\end{definition}

The standard LP relaxation replaces $\sigma_{\mathcal{C}}(u,v) \in \{0,1\}$ with $x_{uv} \in [0,1]$:
\begin{align}
\text{minimize}\quad & \sum_{uv \in E^+} w_{uv} \cdot x_{uv} + \sum_{uv \in E^-} w_{uv} \cdot (1 - x_{uv}) \tag{CC-LP} \label{eq:cc-lp} \\
\text{subject to}\quad & x_{uv} \leq x_{uw} + x_{vw} \quad \forall \text{ distinct } u,v,w \in V, \notag \\
& x_{uv} \in [0,1] \quad \forall \{u,v\} \subseteq V. \notag
\end{align}

Let $x^*$ denote an optimal LP solution with value $\LP^* = \LP(x^*)$.
We denote by $\OPT$ the optimal integral clustering cost.
The LP has $\binom{n}{2}$ variables and $3\binom{n}{3}$ triangle inequality constraints (three per unordered triple of distinct vertices, one for each side).

The LP-PIVOT algorithm \cite{AilonCharikarNewman2008, ChawlaMakarychev2015} picks a random pivot $v \in V$, assigns each $u \in V \setminus \{v\}$ to $v$'s cluster with probability $1 - f(x^*_{uv})$ where $f$ is a rounding function, and recurses on the remaining vertices.
For the specific rounding functions of \cite{FanLeeLee2025}, LP-PIVOT achieves $\E[\obj(\mathcal{C})] \leq \frac{10}{3} \, \LP(x)$ with Lipschitz constant $L = 2$ (Lemma~\ref{lem:baseline-pivot}; see Appendix~\ref{app:prelim} for the full statement and VC dimension background).

\section{Edge Coresets for Correlation Clustering}\label{sec:coresets}

Our first set of results concerns how many edges must be sampled to approximately preserve the cost of every clustering.
We formalize this via the notion of an \emph{edge coreset}.

\begin{definition}[Additive edge coreset]\label{def:coreset}
An \emph{additive $\varepsilon$-edge coreset} for a CC instance $(V, E^+, E^-, w)$ with total weight $W = \sum_e w_e$ is a weighted subgraph $H = (V, E_H, w_H)$ such that for every clustering $\mathcal{C}$ of $V$:
\[
|\obj_H(\mathcal{C}) - \obj_G(\mathcal{C})| \leq \varepsilon W.
\]
\end{definition}

The key step is bounding the VC dimension of the disagreement hypothesis class
$\mathcal{H}_G = \{h_{\mathcal{C}} : E \to \{0,1\} \mid \mathcal{C} \text{ is a partition of } V\}$
where $h_{\mathcal{C}}(uv) = \mathbf{1}[uv \text{ is a disagreement under } \mathcal{C}]$.

\begin{theorem}[VC dimension of clustering disagreements]\label{thm:vc-dim}
For any signed complete graph $G$ on $n \geq 3$ vertices,
$\VC(\mathcal{H}_G) = n - 1$.
\end{theorem}

\begin{proof}[Proof sketch]
For the lower bound, the $n-1$ star edges incident to a fixed vertex $v_1$ are shattered: for any desired disagreement pattern, placing each $v_i$ with $v_1$ or as a singleton realizes it.
For the upper bound, any set of $n$ edges on $n$ vertices contains a cycle; the transitivity of ``same cluster'' forces forbidden disagreement patterns on cycle edges, preventing shattering.
The full proof is in Appendix~\ref{app:vc-proof}.
\end{proof}

\begin{theorem}[Additive edge coreset for CC]\label{thm:coreset}
Let $G = (V, E^+ \uplus E^-, w)$ be a weighted CC instance with $n$ vertices and total weight $W = \sum_e w_e$.
Let $H$ be obtained by sampling $m$ edges independently with probability $p(e) = w_e / W$, assigning weight $W / m$ to each sampled edge (summing weights over repeats).

If $m = \Omega\left(\frac{(n-1) \log(1/\varepsilon) + \log(1/\delta)}{\varepsilon^2}\right) = \tilde{O}(n/\varepsilon^2)$, then $H$ is an additive $\varepsilon$-edge coreset with probability $\geq 1 - \delta$.
\end{theorem}

The proof applies standard VC uniform convergence to $\mathcal{H}_G$; see Appendix~\ref{app:coreset-proof}.
For converting the additive guarantee to an approximation bound: running an $\alpha$-approximation on $H$ yields cost $\leq \alpha \cdot \OPT + (\alpha+1)\varepsilon W$ on $G$, meaningful when $\OPT = \Omega(\varepsilon W)$.
We also obtain relative $\varepsilon$-coresets above a threshold $\tau$ via sensitivity sampling (Theorem~\ref{thm:sensitivity-coreset}; see Appendix~\ref{app:sensitivity}).

\section{Constraint Sparsification}\label{sec:constraint}

Having addressed edge sparsification, we now turn to a complementary question: how many of the $\Theta(n^3)$ triangle inequality constraints are actually needed to define the LP optimum?

\begin{theorem}[Active triangle inequalities at a vertex]\label{thm:constraint-sparsify}
Let $x^*$ be an optimal \emph{vertex} solution to \eqref{eq:cc-lp}, and let $d=\binom{n}{2}$.
Then among the triangle inequalities that are tight at $x^*$, there exists a subset of at most $d$ that is linearly independent.
In particular, there exists a set of at most $d$ tight constraints (triangle inequalities plus bounds) that certifies $x^*$ as a vertex of the full feasible polytope.
\end{theorem}

The proof follows directly from LP theory (a vertex in $\R^d$ is defined by $d$ linearly independent active constraints); see Appendix~\ref{app:constraint-proof}.

\begin{lemma}[Cutting-plane correctness]\label{lem:cutting-plane}
Let $\widehat{\mathcal{T}}$ be any subset of triangle constraints.
Let $\hat{x}$ be the optimum of CC-LP restricted to $\widehat{\mathcal{T}}$ (plus all bound constraints).
If $\hat{x}$ violates some omitted triangle inequality, adding any violated inequality and re-solving refines the feasible region, and the process terminates in finitely many steps with the exact optimum $x^*$.
\end{lemma}

The proof and the full cutting-plane algorithm (\textsc{Sparse-LP-CC}, Algorithm~\ref{alg:sparse-lp}) are in Appendix~\ref{app:cutting-plane}.
Empirically, the cutting-plane LP uses $< 5\%$ of all triangle constraints (Section~\ref{sec:experiments}).

\section{LP-PIVOT on Sparse Instances}\label{sec:sparse-pivot}

We now address the central algorithmic question: can LP-PIVOT achieve near-optimal approximation when only a sparse subset of LP marginals is available?
The key idea is to \emph{impute} missing LP values using the triangle inequality, leveraging observed edges as witnesses.
We assume edge signs and weights are globally known, but LP marginals are \emph{sparsified}: the rounding algorithm observes a random subset $S \subseteq \binom{V}{2}$ of $m$ edges and may read LP marginals $\{x_{uv} : uv \in S\}$ only (see Appendix~\ref{app:access-model} for the formal access model definitions).

\begin{definition}[Triangle imputation]\label{def:impute}
For an unobserved edge $uv$, and a witness $w$ with $uw, vw \in S$, define the \emph{witness interval}:
\[
L_w(uv) = |x_{uw} - x_{vw}|, \qquad U_w(uv) = \min\{x_{uw} + x_{vw},\, 1\},
\]
and the \emph{witness midpoint} $m_w(uv) = (L_w(uv) + U_w(uv))/2$.
Among all observed witnesses, let $w^* = \arg\min_{w} (U_w(uv) - L_w(uv))$ be the witness with the \emph{narrowest} interval.
The \emph{imputed LP value} is:
\begin{equation}\label{eq:impute}
\tilde{x}_{uv} = m_{w^*}(uv) = \frac{L_{w^*}(uv) + U_{w^*}(uv)}{2}.
\end{equation}
If no witnesses exist for $uv$, set $\tilde{x}_{uv} = 1/2$ as a default.
Ties in $\arg\min$ are broken arbitrarily.
\end{definition}

Any feasible $x$ satisfies $x_{uv} \in [L_w, U_w]$ by the triangle inequality, so the midpoint approximates $x_{uv}$ within half the interval width (Lemma~\ref{lem:imputation}; Appendix~\ref{app:imputation}).

\subsection{Witness Density Threshold}

\begin{lemma}[Witness density threshold]\label{lem:witness-density}
Let $S \sim G(n,p)$, i.e., each edge is included independently with probability $p$.
If $p \geq c \sqrt{\log n / n}$ for a sufficiently large constant $c$, then with probability $\geq 1 - n^{-2}$, every pair $(u,v)$ has at least one witness.
Conversely, if $p = o(1/\sqrt{n})$ (equivalently $\E[|S|] = o(n^{3/2})$), most pairs have zero witnesses.
\end{lemma}

This establishes $\tilde{\Theta}(n^{3/2})$ as a sharp phase transition for triangle-based imputation.
The proof is in Appendix~\ref{app:witness-proof}.

\subsection{Good-Witness Condition and Approximation Guarantees}

\begin{definition}[Good-witness fraction]\label{def:good-witness}
For a feasible LP solution $x$ and a parameter $\gamma > 0$, a vertex $w$ is a \emph{$\gamma$-good witness} for pair $(u,v)$ if the witness interval has width $\leq \gamma$.
The instance satisfies the \emph{$(\rho, \gamma)$-good-witness condition} if for every pair $(u,v)$, at least a $\rho$-fraction of vertices $w \in V \setminus \{u,v\}$ are $\gamma$-good witnesses.
\end{definition}

\begin{algorithm}[t]
\caption{\textsc{Sparse-LP-Pivot}}
\label{alg:sparse-pivot}
\begin{algorithmic}[1]
\Require Vertex set $V$, observed edges $S$ with signs/weights and LP marginals $\{x_{uv}\}_{uv \in S}$, rounding functions $f^+, f^-$
\Ensure Clustering $\mathcal{C}$ of $V$
\State Pick pivot $v \in V$ uniformly at random
\State $C \leftarrow \{v\}$
\For{$u \in V \setminus \{v\}$}
    \If{$vu \in S$} \Comment{Observed: sign/weight from instance, LP marginal from $S$}
        \State $\hat{x}_{vu} \leftarrow x_{vu}$
    \Else \Comment{Impute LP marginal; sign/weight from instance}
        \State $\hat{x}_{vu} \leftarrow$ triangle imputation (Definition~\ref{def:impute})
    \EndIf
    \State $p_{vu} \leftarrow f^{\text{sign}_{vu}}(\hat{x}_{vu})$ \Comment{sign globally known}
    \State With probability $1 - p_{vu}$: $C \leftarrow C \cup \{u\}$
\EndFor
\State \Return $\{C\} \cup \textsc{Sparse-LP-Pivot}(V \setminus C, S|_{V \setminus C})$
\end{algorithmic}
\end{algorithm}

\begin{theorem}[Sparse LP-PIVOT under good-witness condition]\label{thm:sparse-pivot}
Let $G$ be a weighted CC instance on $n$ vertices, and let $x$ be any feasible solution to \eqref{eq:cc-lp} satisfying the $(\rho, \gamma)$-good-witness condition (Definition~\ref{def:good-witness}).
If $|S| = \Omega(n^{3/2} \sqrt{\log n / \rho})$ edges are sampled uniformly, then \textsc{Sparse-LP-Pivot} using $x$ with $L$-Lipschitz rounding functions having baseline ratio $\alpha$ (Lemma~\ref{lem:baseline-pivot}) achieves:
\[
\E[\obj(\mathcal{C})] \leq \alpha\,\LP(x) + O(L\gamma) \cdot W.
\]
Instantiating with $\alpha = 10/3$ and $L = 2$ and setting $x = x^*$ gives $\E[\obj(\mathcal{C})] \leq \frac{10}{3}\,\OPT + O(\gamma) \cdot W$.
\end{theorem}

The proof (Appendix~\ref{app:sparse-pivot-proof}) extends the triple-based analysis of \cite{FanLeeLee2025}: witness coverage ensures imputation error $\leq \gamma/2$ per edge, the Lipschitz bound propagates this to rounding probabilities, and summing over triples yields the additive $O(L\gamma) W$ term.

For instances where the good-witness condition may fail, we give a robust bound requiring no structural assumptions:

\begin{definition}[Imputation quality statistic]\label{def:gamma-bar}
For a sample $S$ and feasible LP solution $x$, define the \emph{per-pair imputation width}:
\[
\Gamma(u,v) := \min_{w:\, uw,vw \in S} \big(U_w(uv) - L_w(uv)\big),
\]
with $\Gamma(u,v) = 1$ if $(u,v)$ has no witness.
The \emph{weighted average imputation width} is:
\[
\overline{\Gamma}_w := \frac{1}{W} \sum_{u < v} w_{uv} \, \Gamma(u,v).
\]
\end{definition}

\begin{theorem}[Robust Sparse LP-PIVOT]\label{thm:robust-sparse-pivot}
Let $G$ be a weighted CC instance on $n$ vertices, and let $x$ be any feasible solution to \eqref{eq:cc-lp}.
If $|S| = \Omega(n^{3/2}\sqrt{\log n})$ edges are sampled, then \textsc{Sparse-LP-Pivot} achieves:
\[
\E[\obj(\mathcal{C})] \leq \alpha\,\LP(x) + O(L) \cdot W \cdot \overline{\Gamma}_w.
\]
No assumption on $(\rho,\gamma)$ is needed.
When additionally the $(\rho,\gamma)$-good-witness condition holds, $\overline{\Gamma}_w \leq \gamma$ and this recovers Theorem~\ref{thm:sparse-pivot}.
\end{theorem}

The statistic $\overline{\Gamma}_w$ is computable from $S$ and $x$ \emph{before} running the pivot, serving as a practical diagnostic.
We also prove the good-witness condition holds for tree-metric LP solutions (Proposition~\ref{prop:tree-metric}).
Proofs of both theorems and the tree-metric result are in Appendix~\ref{app:sparse-pivot-proof}.

\section{Lower Bound for General Weighted CC}\label{sec:lower-bound}

The preceding sections show that pseudometric structure enables effective sparsification.
We now prove that without this structure, sparsification fundamentally fails.

\begin{theorem}[Lower bound for general weighted CC]\label{thm:lower-bound}
For general weighted CC on $n$ vertices in the uniform sampling model, any algorithm with $m = o(n)$ samples has expected approximation ratio $\omega(1)$ on some family of instances.
\end{theorem}

\begin{proof}[Proof sketch]
We construct two distributions via a ``hidden heavy negative clique.''
Let $k = \lfloor\sqrt{n}\rfloor$ and $M = n^2$.
$\mathcal{D}_0$: all edges positive (weight 1) except a random matching (negative, weight 1); $\OPT_0 = \Theta(n)$.
$\mathcal{D}_1$: same, but a random $k$-subset $K$ has internal edges negative with weight $M$; $\OPT_1 = \Theta(n^{3/2})$.
With $m = o(n)$ samples, w.h.p.\ no $K$-internal or matching edge is observed, so the two distributions are indistinguishable (Lemma~\ref{lem:indistinguish}).
Any algorithm achieving $O(1)$-approximation on $\mathcal{D}_0$ cuts $O(n)$ pairs, leaving $\Theta(n)$ uncut $K$-internal edges under $\mathcal{D}_1$, incurring cost $\Theta(n^3)$ vs.\ $\OPT_1 = \Theta(n^{3/2})$ (Lemma~\ref{lem:approx-gap}).
By Yao's minimax principle, the ratio is $\omega(1)$.
The full proof with all lemmas is in Appendix~\ref{app:lower-bound-proof}.
\end{proof}

\section{Experimental Validation}\label{sec:experiments}

We validate our theoretical predictions and compare against standard baselines on synthetic and real-world instances.
Our experiments address four questions:
(i) do additive coresets and cutting-plane LP behave as predicted?
(ii) does the $n^{3/2}$ witness threshold manifest empirically?
(iii) how does \textsc{Sparse-LP-Pivot} compare to existing CC algorithms at various sample budgets?
(iv) does the robust statistic $\overline{\Gamma}_w$ reliably predict performance?

We use seven families of instances spanning structured and unstructured regimes:
synthetic pseudometric ($n$ up to $500$), synthetic general, SBM ($k=5$), Political Blogs ($n = 1{,}490$) \cite{AdamicGlance2005}, Facebook ego-network ($n = 4{,}039$) \cite{snapnets}, hidden clique ($n$ up to $400$), and metric violation sweep.
We compare against PIVOT \cite{AilonCharikarNewman2008}, KwikCluster \cite{AilonCharikarNewman2008}, Full LP-PIVOT, and uniform random.
CC-LP is solved via cutting planes (Algorithm~\ref{alg:sparse-lp}) using PuLP/CBC.
All experiments use 20 repetitions.
Full dataset descriptions, per-experiment protocols, and additional figures/tables are in Appendix~\ref{app:experiments}.

\begin{figure}[t]
\centering
\includegraphics[width=0.65\textwidth]{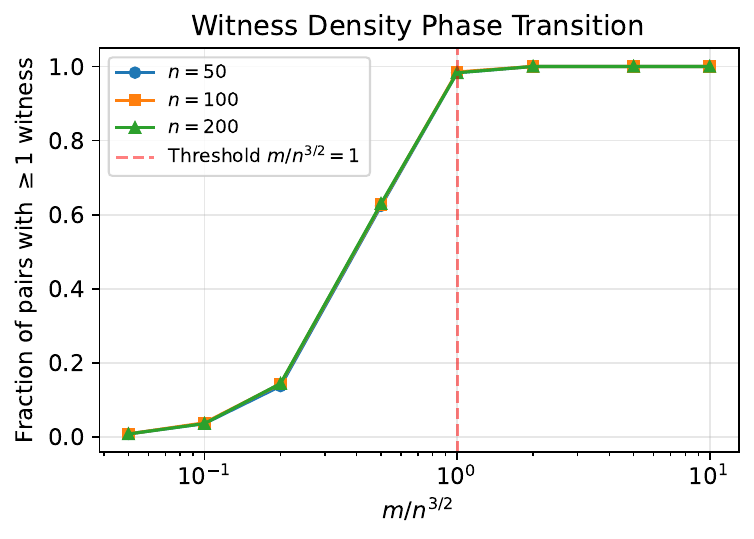}
\caption{Witness density phase transition. The fraction of vertex pairs with $\geq 1$ witness exhibits a sharp threshold at $m/n^{3/2} \approx 1$ (red dashed line), independent of $n$.}
\label{fig:witness}
\end{figure}

\begin{figure}[t]
\centering
\includegraphics[width=0.65\textwidth]{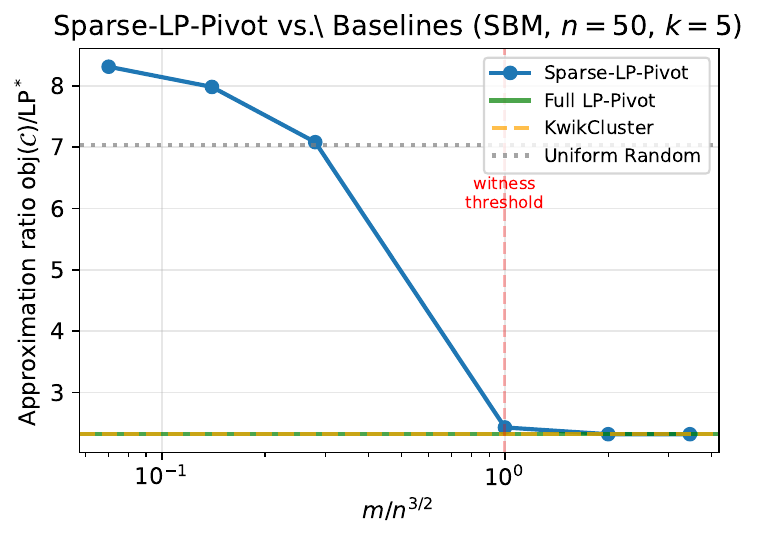}
\caption{Approximation ratio vs.\ sample budget on SBM ($n=50$, $k=5$). \textsc{Sparse-LP-Pivot} converges to Full LP-PIVOT at the witness threshold $m/n^{3/2} \approx 1$.}
\label{fig:baselines}
\end{figure}

The witness density experiment (Figure~\ref{fig:witness}) confirms the sharp $\tilde{\Theta}(n^{3/2})$ phase transition: the fraction of pairs with $\geq 1$ witness jumps from near zero to $>98\%$ at $m/n^{3/2} \approx 1$, with curves nearly identical across $n \in \{50, 100, 200, 500\}$.

The baseline comparison (Figure~\ref{fig:baselines}, Table~\ref{tab:baselines}) shows that \textsc{Sparse-LP-Pivot} converges to Full LP-PIVOT quality at $m \approx n^{3/2}$, precisely at the witness threshold.
Below this threshold, $\overline{\Gamma}_w > 0.7$ and the sparse method degrades significantly.
KwikCluster achieves a comparable ratio without LP marginals on these unweighted instances; the LP advantage is more pronounced on weighted instances with continuous solutions.

\begin{table}[t]
\centering
\caption{Approximation ratio $\obj(\mathcal{C})/\LP^*$ on SBM instance ($n=50$, $k=5$, 15 trials).}
\label{tab:baselines}
\begin{tabular}{lccccc}
\toprule
$m$ & $m/n^{3/2}$ & Sparse-LP-Pivot & Full LP-Pivot & KwikCluster & $\overline{\Gamma}_w$ \\
\midrule
25  & 0.07 & 8.31 & 2.32 & 2.32 & 0.981 \\
50  & 0.14 & 7.98 & 2.32 & 2.32 & 0.930 \\
100 & 0.28 & 7.08 & 2.32 & 2.32 & 0.731 \\
353 & 1.00 & 2.43 & 2.32 & 2.32 & 0.017 \\
706 & 2.00 & 2.32 & 2.32 & 2.32 & 0.000 \\
\bottomrule
\end{tabular}
\end{table}

Additional experiments (Appendix~\ref{app:experiments}) confirm: additive coresets achieve $\varepsilon$-error at the predicted $\tilde{O}(n/\varepsilon^2)$ sample size; the cutting-plane LP uses $<5\%$ of triangle constraints; $\overline{\Gamma}_w$ reliably predicts approximation quality across all datasets; and the hidden-clique construction produces unbounded ratios at $o(n)$ samples.

\section{Conclusion}

We have established a structural dichotomy for the sparsifiability of correlation clustering.
For pseudometric instances, the triangle inequality propagates information across unobserved edges, enabling additive $\varepsilon$-coresets of optimal size $\tilde{O}(n/\varepsilon^2)$, exact LP solving via a small active constraint set, and near-$\frac{10}{3}$-approximate rounding from $\tilde{O}(n^{3/2})$ sampled LP marginals.
For general weighted instances, no such propagation exists, and we prove that $o(n)$ uniform samples are provably insufficient for constant-factor approximation.
The pseudometric condition thus governs not only the tractability of CC but also its robustness to incomplete information.

Several natural questions remain open.
The gap between our $\tilde{O}(n^{3/2})$ upper bound for sparse rounding and the $\Omega(n)$ lower bound for general CC leaves room for a tighter characterization in the pseudometric regime.
Our good-witness condition is verified for tree metrics (Proposition~\ref{prop:tree-metric}) and validated empirically on Euclidean instances, but a proof for general pseudometric or near-metric instances remains elusive.
Finally, whether multiplicative $(1 \pm \varepsilon)$-coresets of size $\tilde{O}(n)$ exist under structural assumptions, and whether the sparsification landscape changes for chromatic correlation clustering \cite{Bonchi2012, FanLeeLee2025}, are directions we find particularly compelling.

\bibliographystyle{plainnat}
\bibliography{cc_sparsification}

\newpage
\appendix

\section{Deferred Preliminaries}\label{app:prelim}

\subsection{Proof Techniques Overview}

The cost of a clustering $\mathcal{C}$ is a weighted sum over edges: $\obj(\mathcal{C}) = \sum_{e \in E} w_e \cdot \mathbf{1}[e \text{ disagrees with } \mathcal{C}]$.
Our coreset result begins by analyzing the \emph{VC dimension} of the hypothesis class $\mathcal{H} = \{h_{\mathcal{C}} : E \to \{0,1\} \mid \mathcal{C} \text{ is a partition}\}$ where $h_{\mathcal{C}}(e) = \mathbf{1}[e \text{ is a disagreement}]$.
We prove $\VC(\mathcal{H}) = n - 1$ exactly, which by standard uniform convergence yields \emph{additive} $\varepsilon$-coresets of size $\tilde{O}(n / \varepsilon^2)$: for every clustering $\mathcal{C}$, $|\obj_H(\mathcal{C}) - \obj_G(\mathcal{C})| \leq \varepsilon W$.
For relative (multiplicative) coresets, we use sensitivity-based sampling; the total sensitivity $\mathfrak{S} = \sum_e s(e) \leq W / \OPT$ always holds, and when $\OPT$ is bounded away from zero this yields efficient relative coresets.

Turning to the LP, we observe that constraints with positive slack at the optimum have zero dual variables and can be removed without affecting the LP value.
At any LP vertex solution, at most $\binom{n}{2}$ constraints (including bounds) are linearly independent among the active set, so at most $\binom{n}{2}$ triangle inequalities appear in any linearly independent active subset.
We give a cutting-plane algorithm that correctly identifies the active constraints with finite termination.

For the rounding phase, LP-PIVOT processes edges incident to each pivot, and when some edges are missing we \emph{impute} their LP values using the triangle inequality: $x^*_{uv} \in [|x^*_{uw} - x^*_{wv}|, x^*_{uw} + x^*_{wv}]$ for any observed vertex $w$.
A critical subtlety is that with $|S|$ uniformly sampled edges, a pair $(u,v)$ has witnesses $w$ (with both $(u,w)$ and $(v,w)$ in $S$) only when $|S| = \tilde{\Omega}(n^{3/2})$.
We prove this threshold is necessary and sufficient, and give an approximation guarantee conditional on a \emph{good-witness fraction} assumption.

Finally, for the lower bound we construct a ``hidden heavy negative clique'' instance: $n$ vertices with uniform positive weights except for a random subset $K$ of size $\sqrt{n}$ whose internal edges are negative with large weight $M$.
Via Yao's minimax principle, we prove that $o(n)$ uniform samples cannot distinguish this from the all-positive instance, leading to an unbounded approximation ratio.

\subsection{LP-PIVOT Algorithm}

The LP-PIVOT algorithm \cite{AilonCharikarNewman2008, ChawlaMakarychev2015} picks a random pivot $v \in V$, assigns each $u \in V \setminus \{v\}$ to $v$'s cluster with probability $1 - f(x^*_{uv})$ where $f$ is a rounding function, and recurses on the remaining vertices.
Its analysis uses a \emph{triple-based} framework: for each triple $(u,v,w)$, the expected algorithmic cost is bounded by $\alpha$ times the LP cost contribution.

\begin{lemma}[Baseline LP-PIVOT guarantee]\label{lem:baseline-pivot}
Fix $L$-Lipschitz rounding functions $f^+, f^- : [0,1] \to [0,1]$.
The corresponding LP-PIVOT algorithm (with exact access to all LP marginals of a feasible solution $x$) satisfies
$\E[\obj(\mathcal{C})] \leq \alpha \, \LP(x)$
for some constant $\alpha$ depending on $f^+, f^-$.
For the specific rounding functions of \cite{FanLeeLee2025}, $\alpha = 10/3$ and $L = 2$.
\end{lemma}

\subsection{VC Dimension and Uniform Convergence}

\begin{definition}[VC dimension]
Let $\mathcal{H}$ be a set system over ground set $\mathcal{X}$.
A set $S \subseteq \mathcal{X}$ is \emph{shattered} by $\mathcal{H}$ if $\{S \cap h : h \in \mathcal{H}\} = 2^S$.
The \emph{VC dimension} $\VC(\mathcal{H})$ is the size of the largest shattered set.
\end{definition}

\begin{theorem}[Uniform convergence, \cite{VapnikChervonenkis1971}]\label{thm:vc-convergence}
Let $\mathcal{H}$ be a hypothesis class with $\VC(\mathcal{H}) = d$.
If $S$ is a set of $m \geq c \cdot \frac{d \log(1/\varepsilon) + \log(1/\delta)}{\varepsilon^2}$ i.i.d.\ samples from a distribution $\mathcal{D}$, then with probability $\geq 1 - \delta$:
\[
\sup_{h \in \mathcal{H}} \left| \frac{1}{m}\sum_{s \in S} h(s) - \E_{s \sim \mathcal{D}}[h(s)] \right| \leq \varepsilon.
\]
\end{theorem}

\section{Full Proofs for Edge Coresets}\label{app:coreset-proofs}

\subsection{VC Dimension of Clustering Disagreements}\label{app:vc-proof}

For a signed graph $G = (V, E^+ \uplus E^-)$, define the hypothesis class $\mathcal{H}_G$ over the edge set $E = E^+ \cup E^-$:
\[
\mathcal{H}_G = \{h_{\mathcal{C}} : E \to \{0,1\} \mid \mathcal{C} \text{ is a partition of } V\}
\]
where $h_{\mathcal{C}}(uv) = \mathbf{1}[uv \text{ is a disagreement under } \mathcal{C}]$.

\begin{proof}[Proof of Theorem~\ref{thm:vc-dim}]
\textbf{Lower bound: $\VC(\mathcal{H}_G) \geq n - 1$.}

Let $S = \{v_1 v_i : i = 2, \ldots, n\}$ be the $n-1$ star edges incident to $v_1$ (with their given signs in $G$).
For any labeling $\ell : S \to \{0,1\}$ (interpreted as desired disagreement indicators), we construct a clustering $\mathcal{C}_\ell$ that realizes it.
Initialize a cluster $C$ containing $v_1$.
For each $i \geq 2$:
\begin{itemize}[nosep,leftmargin=2em]
\item If $v_1 v_i \in E^+$: place $v_i$ in $C$ iff $\ell(v_1 v_i) = 0$; otherwise make $v_i$ a singleton.
\item If $v_1 v_i \in E^-$: place $v_i$ in $C$ iff $\ell(v_1 v_i) = 1$; otherwise make $v_i$ a singleton.
\end{itemize}
Then for every star edge $v_1 v_i \in S$, the edge disagrees under $\mathcal{C}_\ell$ iff $\ell(v_1 v_i) = 1$.
Hence $S$ is shattered and $\VC(\mathcal{H}_G) \geq n - 1$.

\textbf{Upper bound: $\VC(\mathcal{H}_G) \leq n - 1$.}

We show that no set of $n$ edges can be shattered.
Let $S = \{e_1, \ldots, e_n\}$ be any set of $n$ edges of $K_n$.
View $S$ as a graph on at most $n$ vertices with $n$ edges.
Since a forest on $n$ vertices has at most $n - 1$ edges, $S$ contains a cycle $C = (e_{i_1}, \ldots, e_{i_k})$ of length $k \geq 3$, visiting distinct vertices $u_1, \ldots, u_k$ with $e_{i_j} = u_j u_{j+1}$ (indices mod $k$).

For any partition $\mathcal{C}$, define $\sigma_j = \mathbf{1}[u_j \text{ and } u_{j+1} \text{ in different clusters}]$.
We claim that $\sigma$ cannot have exactly one nonzero entry.
Indeed, if $\sigma_1 = 1$ and $\sigma_j = 0$ for $j = 2, \ldots, k$, then transitivity along the zero edges gives $u_2 = u_3 = \cdots = u_k = u_1$ (all in the same cluster), contradicting $\sigma_1 = 1$ (i.e., $u_1 \neq u_2$).
By symmetry, the same holds for any single nonzero entry.

Now, the disagreement indicator satisfies $h_{\mathcal{C}}(e_{i_j}) = \sigma_j \oplus \mathbf{1}[e_{i_j} \in E^-]$, where $\oplus$ denotes XOR.
Since this is a bijection between $\sigma$-patterns and $h$-patterns (XOR with a fixed vector), the $k$ forbidden $\sigma$-patterns (exactly one nonzero entry) map to $k$ forbidden $h$-patterns on the cycle edges.

Consider any labeling $\ell : S \to \{0,1\}$ whose restriction to the cycle edges $\ell|_C$ is one of these $k$ forbidden $h$-patterns.
No partition can realize $\ell$, because the cycle constraint is violated regardless of the non-cycle edges.
There are $k \cdot 2^{n-k} \geq 3 \cdot 2^{n-3} > 0$ such unrealizable labelings.
Therefore $|\{h_{\mathcal{C}}|_S : \mathcal{C}\}| \leq 2^n - k \cdot 2^{n-k} < 2^n$, so $S$ is not shattered.
\end{proof}

\begin{remark*}
The tight bound $\VC(\mathcal{H}_G) = n - 1$ improves the naive Bell-number upper bound of $O(n\log n)$ and yields optimal coreset sample complexity $\tilde{O}(n/\varepsilon^2)$ (Theorem~\ref{thm:coreset}).
We verified computationally that $\VC(\mathcal{H}_G) = n - 1$ for all $n \leq 7$.
\end{remark*}

\subsection{Coreset Construction}\label{app:coreset-proof}

\begin{remark*}
The additive formulation is the natural one for VC-based coresets.
A multiplicative $(1 \pm \varepsilon)$ guarantee requires $\obj_G(\mathcal{C}) > 0$ for all $\mathcal{C}$, which fails when $\OPT = 0$ (e.g., all edges positive).
We address relative-error coresets via sensitivity sampling in Appendix~\ref{app:sensitivity}.
\end{remark*}

\begin{proof}[Proof of Theorem~\ref{thm:coreset}]
For any clustering $\mathcal{C}$:
\[
\obj_G(\mathcal{C}) = \sum_e w_e \, h_{\mathcal{C}}(e) = W \cdot \E_{e \sim p}[h_{\mathcal{C}}(e)]
\]
where $p(e) = w_e / W$.
The coreset estimator is:
\[
\obj_H(\mathcal{C}) = \frac{W}{m} \sum_{i=1}^{m} h_{\mathcal{C}}(e_i) = W \cdot \frac{1}{m}\sum_{i=1}^{m} h_{\mathcal{C}}(e_i).
\]
Therefore:
\[
\frac{1}{W} |\obj_H(\mathcal{C}) - \obj_G(\mathcal{C})| = \left|\frac{1}{m}\sum_{i=1}^{m} h_{\mathcal{C}}(e_i) - \E_{e \sim p}[h_{\mathcal{C}}(e)]\right|.
\]

By Theorem~\ref{thm:vc-convergence} applied to $\mathcal{H}_G$ with $\VC(\mathcal{H}_G) = n - 1$ (Theorem~\ref{thm:vc-dim}), the stated sample size $m$ guarantees:
\[
\sup_{\mathcal{C}} \left| \frac{1}{m}\sum_{i=1}^{m} h_{\mathcal{C}}(e_i) - \E_p[h_{\mathcal{C}}(e)] \right| \leq \varepsilon
\]
with probability $\geq 1 - \delta$.
Multiplying both sides by $W$ gives $|\obj_H(\mathcal{C}) - \obj_G(\mathcal{C})| \leq \varepsilon W$ for all $\mathcal{C}$ simultaneously.
\end{proof}

\begin{remark*}
For converting the additive guarantee to an approximation bound: if one runs an $\alpha$-approximation algorithm on $H$ to obtain $\mathcal{C}$, then $\obj_G(\mathcal{C}) \leq \obj_H(\mathcal{C}) + \varepsilon W \leq \alpha \cdot \obj_H(\mathcal{C}^*) + \varepsilon W \leq \alpha \cdot (\obj_G(\mathcal{C}^*) + \varepsilon W) + \varepsilon W = \alpha \cdot \OPT + (\alpha + 1)\varepsilon W$.
This is meaningful when $\OPT = \Omega(\varepsilon W)$, i.e., the instance is not nearly ``trivially clusterable.''
\end{remark*}

\subsection{Sensitivity-Based Coresets for Relative Error}\label{app:sensitivity}

The additive coreset of Theorem~\ref{thm:coreset} preserves costs within $\varepsilon W$.
For a \emph{relative} (multiplicative) guarantee, we use sensitivity sampling.

\begin{definition}[Sensitivity, \cite{FeldmanLangberg2011}]
The \emph{sensitivity} of edge $e$ with respect to weighted CC is:
\[
s(e) = \sup_{\mathcal{C} : \obj_G(\mathcal{C}) > 0} \frac{w_e \cdot h_{\mathcal{C}}(e)}{\obj_G(\mathcal{C})}.
\]
The \emph{total sensitivity} is $\mathfrak{S} = \sum_{e \in E} s(e)$.
\end{definition}

\begin{lemma}[Sensitivity upper bound]\label{lem:pseudometric-sensitivity}
For any weighted CC instance with $\OPT > 0$:
\[
s(e) \leq \frac{w_e}{\OPT} \quad \text{for all } e, \qquad \text{hence} \qquad \mathfrak{S} = \sum_e s(e) \leq \frac{W}{\OPT}.
\]
\end{lemma}

\begin{proof}
For any edge $e$ and any clustering $\mathcal{C}$ with $\obj_G(\mathcal{C}) > 0$:
if $h_{\mathcal{C}}(e) = 0$, the ratio is $0$.
If $h_{\mathcal{C}}(e) = 1$:
\[
\frac{w_e \cdot 1}{\obj_G(\mathcal{C})} \leq \frac{w_e}{\OPT}
\]
since $\obj_G(\mathcal{C}) \geq \OPT$.
Taking the supremum over $\mathcal{C}$ gives $s(e) \leq w_e / \OPT$.
Summing over all edges: $\mathfrak{S} \leq W / \OPT$.
\end{proof}

\begin{remark*}
This bound is tight in general: a single heavy edge $e^*$ with $w_{e^*} = W - o(W)$ that the optimal clustering must disagree with has $s(e^*) \approx W / \OPT$.
The bound $\mathfrak{S} = W / \OPT$ can range from $O(1)$ (when cost is spread evenly) to $\Theta(n^2)$ (when $\OPT$ is very small relative to $W$).

One might hope that pseudometric weights alone force $\OPT \geq \Omega(W/n)$, giving $\mathfrak{S} = O(n)$.
This is \textbf{false}: consider all edges positive with pseudometric weights, in which case $\OPT = 0$.
Even with mixed signs, the pseudometric condition alone does not prevent $\OPT$ from being arbitrarily small relative to $W$.
Obtaining $\mathfrak{S} = O(n)$ requires stronger structural assumptions, such as a clusterability or margin condition (e.g., $\OPT \geq \Omega(W/n)$ as an explicit assumption).
\end{remark*}

\begin{theorem}[Relative error above a threshold]\label{thm:sensitivity-coreset}
Let $G$ be a weighted CC instance with total weight $W$ and $p(e) = w_e/W$.
Fix parameters $\varepsilon, \delta \in (0,1)$ and a threshold $\tau > 0$.
Sample $m$ edges i.i.d.\ from $p$ and reweight each sampled edge by $W/m$ (summing weights over repeats) to form $H$.
If
\[
m = \Omega\!\left(\frac{d \log(1/\varepsilon) + \log(1/\delta)}{\varepsilon^2} \cdot \frac{W^2}{\tau^2}\right),
\]
where $d = \VC(\mathcal{H}_G) = n - 1$, then with probability $\geq 1 - \delta$, simultaneously for all clusterings $\mathcal{C}$:
\[
\obj_G(\mathcal{C}) \geq \tau \quad \Longrightarrow \quad |\obj_H(\mathcal{C}) - \obj_G(\mathcal{C})| \leq \varepsilon \, \obj_G(\mathcal{C}).
\]
Equivalently, for all $\mathcal{C}$: $|\obj_H(\mathcal{C}) - \obj_G(\mathcal{C})| \leq \varepsilon \, \max\{\obj_G(\mathcal{C}),\, \tau\}$.
\end{theorem}

\begin{proof}
By Theorem~\ref{thm:coreset}, with $m = \Omega((d\log(1/\varepsilon') + \log(1/\delta))/\varepsilon'^2)$ we have $|\obj_H(\mathcal{C}) - \obj_G(\mathcal{C})| \leq \varepsilon' W$ for all $\mathcal{C}$ simultaneously.
Set $\varepsilon' := \varepsilon \tau / W$.
Then for any $\mathcal{C}$ with $\obj_G(\mathcal{C}) \geq \tau$:
\[
|\obj_H(\mathcal{C}) - \obj_G(\mathcal{C})| \leq \varepsilon' W = \varepsilon \tau \leq \varepsilon \, \obj_G(\mathcal{C}).
\]
The sample size becomes $m = \Omega((d\log(W/(\varepsilon\tau)) + \log(1/\delta)) \cdot W^2 / (\varepsilon^2 \tau^2))$, which is $\tilde{O}(d \cdot W^2 / (\varepsilon^2 \tau^2))$.
\end{proof}

\begin{remark*}[Why clean multiplicative coresets are delicate]
A uniform multiplicative guarantee $|\obj_H(\mathcal{C}) - \obj_G(\mathcal{C})| \leq \varepsilon \, \obj_G(\mathcal{C})$ for all clusterings is impossible when the objective can be arbitrarily small (even if nonzero), since any fixed additive estimation error becomes unbounded in relative terms.
Theorem~\ref{thm:sensitivity-coreset} therefore provides a relative guarantee only above a user-chosen threshold $\tau$.
Setting $\tau = \OPT$ (if known or estimated) recovers a multiplicative guarantee for near-optimal clusterings.
Obtaining true multiplicative coresets under additional structural assumptions (e.g., margin conditions implying $\obj_G(\mathcal{C}) \geq \Omega(W/n)$ for all relevant $\mathcal{C}$) is an interesting open direction.
\end{remark*}

\section{Full Proofs for Constraint Sparsification}\label{app:constraint-proof}

\begin{proof}[Proof of Theorem~\ref{thm:constraint-sparsify}]
A vertex (basic feasible solution) of a linear program with $d$ variables is defined by $d$ linearly independent active constraints.
CC-LP has $d = \binom{n}{2}$ variables $\{x_{uv}\}$.
At any vertex $x^*$, there exist exactly $d$ linearly independent active constraints drawn from the full constraint set (triangle inequalities plus bound constraints $0 \leq x_{uv} \leq 1$).

Note that the \emph{total} number of active (zero-slack) constraints at a vertex may exceed $d$ due to degeneracy, since multiple constraints can be simultaneously tight.
However, among these active constraints, at most $d$ are linearly independent.
In particular, at most $d$ triangle inequalities can appear in any linearly independent active set.

Since $x^*$ is a vertex (basic feasible solution) of the full CC-LP feasible polytope in $\R^d$,
there exists a set of exactly $d$ linearly independent constraints from the full constraint system that are tight at $x^*$ and whose intersection is $\{x^*\}$.
All constraints in such a certificate are necessarily active (zero slack) at $x^*$.
Therefore, among the tight triangle inequalities at $x^*$ there can be at most $d$ that are linearly independent, and together with bound constraints they form a size-$d$ linearly independent active set certifying that $x^*$ is a vertex of the full polytope.
\end{proof}

\begin{remark*}
Theorem~\ref{thm:constraint-sparsify} guarantees existence of a small active set but does not give an efficient algorithm to find it without solving the full LP.
The following lemma provides a correct algorithmic approach.
\end{remark*}

\subsection{Cutting-Plane Identification of Active Constraints}\label{app:cutting-plane}

\begin{proof}[Proof of Lemma~\ref{lem:cutting-plane}]
The restricted LP relaxes constraints, so its optimum value satisfies $\LP(\widehat{\mathcal{T}}) \leq \LP^*$ (minimization with fewer constraints can only decrease or maintain the objective).
Adding a violated constraint removes points from the feasible region that violate the full LP, so the restricted optimum is non-decreasing: $\LP(\widehat{\mathcal{T}}) \leq \LP(\widehat{\mathcal{T}} \cup \{c\}) \leq \LP^*$ for any violated constraint $c$.

The cutting-plane loop terminates when no violated constraints remain, i.e., $\hat{x}$ is feasible for the full LP.
Since the restricted LP is a relaxation, $\LP(\widehat{\mathcal{T}}) \leq \LP^*$.
But feasibility of $\hat{x}$ for the full LP implies $\LP^* \leq \LP(\hat{x}) = \LP(\widehat{\mathcal{T}})$.
Hence $\LP(\hat{x}) = \LP^*$, so $\hat{x}$ is optimal for the full LP.
Finite termination follows because there are finitely many ($3\binom{n}{3}$) triangle inequalities.
\end{proof}

\subsection{Two-Phase LP Algorithm}

\begin{algorithm}[h]
\caption{\textsc{Sparse-LP-CC}: Cutting-plane LP solver for CC}
\label{alg:sparse-lp}
\begin{algorithmic}[1]
\Require Complete weighted graph $G = (V, E^+, E^-, w)$, separation oracle
\Ensure Exact LP optimum $x^*$
\State $\widehat{\mathcal{T}} \leftarrow \emptyset$; include all bound constraints $0 \leq x_{uv} \leq 1$
\State \textbf{Phase 1: Optional warm start} (heuristic constraint selection)
\For{each triple $(u,v,w)$ with small weight slack $\Delta_w(u,v,w) = w_{uv} + w_{vw} - w_{wu} \leq \gamma$}
    \State $\widehat{\mathcal{T}} \leftarrow \widehat{\mathcal{T}} \cup \{(u,v,w)\}$ \Comment{Heuristic: tight weights suggest tight LP constraints}
\EndFor
\State \textbf{Phase 2: Solve and repair (correct by Lemma~\ref{lem:cutting-plane})}
\Repeat
    \State Solve CC-LP restricted to $\widehat{\mathcal{T}}$ $\to$ solution $\hat{x}$
    \State Run separation oracle: find most violated triangle inequality, if any
    \If{violated constraint found}
        \State $\widehat{\mathcal{T}} \leftarrow \widehat{\mathcal{T}} \cup \{\text{violated constraint}\}$
    \EndIf
\Until{no violated constraints}
\State \Return $\hat{x}$
\end{algorithmic}
\end{algorithm}

The correctness of Algorithm~\ref{alg:sparse-lp} follows from Lemma~\ref{lem:cutting-plane}.
The efficiency depends on how many iterations the cutting-plane loop requires.
The warm start in Phase~1 is heuristic (it uses the intuition that weight-tight triples may correspond to LP-tight triples), but correctness does not depend on this heuristic being accurate.
Empirically, for pseudometric instances, the warm start captures most active constraints and the loop terminates in few iterations.

\section{Full Proofs for Sparse LP-PIVOT}\label{app:sparse-pivot-proofs}

\subsection{Access Model: Sparse Edge Observation and LP Marginals}\label{app:access-model}

Our goal is to isolate the \emph{information} needed for LP-based rounding under sparsification.
Accordingly, we separate (i) obtaining an LP solution (exact or approximate) from (ii) running a rounding algorithm on a sparsified view of the instance.

\begin{definition}[Observed edge data]
For an edge $e = (u,v)$, the observable edge data consists of its sign ($+$ or $-$) and weight $w_{uv}$.
\end{definition}

\begin{definition}[LP-marginal input]
Let $x$ denote any feasible solution to \eqref{eq:cc-lp} (not necessarily optimal), produced by an LP solver.
The rounding procedure may access \emph{LP marginals} $x_{uv}$ for queried pairs $(u,v)$.
In our experiments on moderate $n$ we use an exact optimum $x^*$, while the theory applies verbatim to any feasible $x$ by replacing $\LP^*$ with $\LP(x)$ in the bounds.
\end{definition}

\begin{definition}[$m$-sparse access model for rounding]
The CC instance $(V, E^+, E^-, w)$ is fully specified.
An LP solver produces a feasible solution $x$ to \eqref{eq:cc-lp} (exact or approximate).
Edge signs and weights are globally accessible to the rounding algorithm.
However, LP marginals are \emph{sparsified}: the rounding algorithm is given a random subset $S \subseteq \binom{V}{2}$ of $m$ edges and may read the LP marginals $\{x_{uv} : uv \in S\}$.
For pairs $uv \notin S$, the LP marginal $x_{uv}$ is not directly available and must be imputed.
\end{definition}

\begin{remark*}[Why this model is not circular]
The computational task of obtaining $x$ is orthogonal to the information-theoretic task of rounding from sparse observations.
This paper addresses both aspects: Section~\ref{sec:constraint} shows that an exact optimum can be obtained via cutting planes using only a small active constraint set at an LP vertex, while Sections~\ref{sec:sparse-pivot}--\ref{sec:lower-bound} characterize how many \emph{edges} must be observed for LP-based rounding to remain accurate once LP marginals are available.
\end{remark*}

When LP-PIVOT selects a pivot $v$ and processes a vertex $u$, it requires an LP marginal $x_{vu}$.
If $(v,u) \notin S$, we impute this value using triangle inequalities from witnessed pairs in $S$ (Definition~\ref{def:impute}).
Our analysis bounds the degradation in rounding performance as a function of (a) the witness density induced by $S$ and (b) the imputation quality of $x$ on those witnessed triangles.

\subsection{Imputation Accuracy}\label{app:imputation}

\begin{lemma}[Imputation accuracy]\label{lem:imputation}
For any feasible LP solution $x$, any unobserved edge $uv$, and any witness $w$, the true LP value satisfies $x_{uv} \in [L_w(uv),\, U_w(uv)]$.
If the witness interval has width $U_w(uv) - L_w(uv) \leq \gamma$, then the midpoint $(L_w + U_w)/2$ approximates $x_{uv}$ within $\gamma/2$.
\end{lemma}

\begin{proof}
The bounds $|x_{uw} - x_{vw}| \leq x_{uv} \leq x_{uw} + x_{vw}$ follow directly from the triangle inequality constraints of CC-LP, which any feasible $x$ satisfies.
Since $x_{uv} \leq 1$, we also have $x_{uv} \leq \min\{x_{uw}+x_{vw}, 1\} = U_w(uv)$.
Thus $x_{uv} \in [L_w, U_w]$.
The midpoint of any interval of width $\leq \gamma$ containing $x_{uv}$ is within $\gamma/2$ of $x_{uv}$.
\end{proof}

\subsection{Witness Density}\label{app:witness-proof}

\begin{proof}[Proof of Lemma~\ref{lem:witness-density}]
For a fixed pair $(u,v)$, $X_{uv} \sim \mathrm{Bin}(n-2, p^2)$ with mean $\mu = (n-2) p^2$.
\[
\Pr[X_{uv} = 0] = (1-p^2)^{n-2} \leq \exp(-(n-2)p^2) = e^{-\mu}.
\]
If $p^2 \geq 6\log n / n$, then $\mu \geq 6(n-2)\log n / n \geq 4\log n$ for large $n$, so $\Pr[X_{uv} = 0] \leq n^{-4}$.
By a union bound over $\binom{n}{2} \leq n^2$ pairs: $\Pr[\exists (u,v) \text{ with no witness}] \leq n^2 \cdot n^{-4} = n^{-2}$.

Converting: $p = c\sqrt{\log n / n}$ for a sufficiently large constant $c$ gives $\E[|S|] = p \binom{n}{2} = \Theta(n^{3/2}\sqrt{\log n})$.

For the converse: if $p = o(1/\sqrt{n})$, then $\mu = (n-2)p^2 = o(1)$, and $\E[X_{uv}] \to 0$.
\end{proof}

\begin{remark*}
This is a fundamental barrier: the $\tilde{O}(n)$ sample claim in the original formulation of this work was incorrect.
Triangle imputation \emph{requires} $\tilde{\Omega}(n^{3/2})$ uniform samples just to have witnesses.
This is the correct sample complexity for the imputation approach.

\emph{Sampling model note:} Lemma~\ref{lem:witness-density} is stated for $G(n,p)$ (independent Bernoulli edges).
For sampling $m$ distinct edges uniformly without replacement, the same conclusion holds by a standard coupling: the hypergeometric distribution (sampling without replacement) is more concentrated than the corresponding binomial (sampling with replacement), so all tail bounds carry over.
\end{remark*}

\subsection{Good-Witness Condition}

\begin{lemma}[Min-width imputation under good-witness condition]\label{lem:minwidth-imputation}
Let $x$ be any feasible LP solution, let $S \sim G(n,p)$ (each edge independently present with probability $p$), and suppose the $(\rho, \gamma)$-good-witness condition holds with respect to $x$ and $\rho > 0$.
If
\[
n p^2 \geq \frac{c \, \log n}{\rho},
\]
then with probability $\geq 1 - n^{-2}$, the min-width imputation (Definition~\ref{def:impute}) satisfies $|\tilde{x}_{uv} - x_{uv}| \leq \gamma/2$ for all pairs $(u,v) \notin S$.
\end{lemma}

\begin{proof}
Fix a pair $(u,v)$.
Among the $n-2$ potential witnesses, at least $\rho(n-2)$ are $\gamma$-good.
Each potential witness $w$ becomes an \emph{observed} witness iff both $(u,w)$ and $(v,w)$ are in $S$, which occurs independently with probability $p^2$.

Let $Y_g \sim \mathrm{Bin}(\lceil\rho(n-2)\rceil, p^2)$ be the number of observed good witnesses.
We have $\E[Y_g] \geq \rho(n-2)p^2 \geq \rho n p^2 / 2$ for large $n$.
If $np^2 \geq c\log n / \rho$, then $\E[Y_g] \geq c\log n / 2$.
By a Chernoff bound:
\[
\Pr[Y_g = 0] \leq (1 - p^2)^{\rho(n-2)} \leq \exp(-\rho(n-2)p^2) \leq \exp(-c\log n / 2) \leq n^{-4}
\]
for $c \geq 8$.

Conditioned on $Y_g \geq 1$, at least one observed witness $w$ is $\gamma$-good, meaning $U_w(uv) - L_w(uv) \leq \gamma$.
The min-width estimator selects $w^* = \arg\min_w (U_w - L_w)$, so $U_{w^*} - L_{w^*} \leq \gamma$.
By Lemma~\ref{lem:imputation}, $x_{uv} \in [L_{w^*}, U_{w^*}]$.
The midpoint $m_{w^*} = (L_{w^*} + U_{w^*})/2$ of an interval of width $\leq \gamma$ containing $x_{uv}$ satisfies:
\[
|\tilde{x}_{uv} - x_{uv}| = |m_{w^*} - x_{uv}| \leq \frac{U_{w^*} - L_{w^*}}{2} \leq \frac{\gamma}{2}.
\]

Union bound over all $\binom{n}{2} \leq n^2$ pairs gives total failure probability $\leq n^2 \cdot n^{-4} = n^{-2}$.
\end{proof}

\subsection{Approximation Guarantee}\label{app:sparse-pivot-proof}

\begin{proof}[Proof of Theorem~\ref{thm:sparse-pivot}]
The proof extends the triple-based analysis of \cite{FanLeeLee2025}.

\emph{Step 1: Witness coverage.}
By Lemma~\ref{lem:witness-density}, $|S| = \Omega(n^{3/2}\sqrt{\log n})$ ensures every pair has at least one witness w.h.p.
By Lemma~\ref{lem:minwidth-imputation}, under the $(\rho, \gamma)$-good-witness condition, the imputation satisfies $|\tilde{x}_{uv} - x_{uv}| \leq \gamma/2$ for all pairs w.h.p.

\emph{Step 2: Perturbation of rounding probabilities.}
Since the rounding functions of \cite{FanLeeLee2025} have Lipschitz constant $L$:
\[
|f(\tilde{x}_{vu}) - f(x_{vu})| \leq L \cdot |\tilde{x}_{vu} - x_{vu}| \leq L\gamma/2.
\]

\emph{Step 3: Triple-based analysis.}
For a triple $(u,v,w)$ with pivot $v$, the expected algorithmic cost under \textsc{Sparse-LP-Pivot} satisfies:
\[
\E[\ALG(uvw)] = \E[\ALG^{\mathrm{exact}}(uvw)] + O(L\gamma) \cdot (\text{total weight contribution of } uvw)
\]
where $\ALG^{\mathrm{exact}}$ is the cost of standard LP-PIVOT using $x$ with exact marginals (bounded by $\alpha \, \LP(x)$ via Lemma~\ref{lem:baseline-pivot}).
The $O(L\gamma)$ perturbation comes from the Lipschitz bound on each rounding probability (with the factor of $1/2$ from Lemma~\ref{lem:minwidth-imputation} absorbed into the constant).

\emph{Step 4: Summing over triples.}
\begin{align*}
\E[\ALG] &= \sum_{\text{triples}} \frac{1}{n}\E[\ALG(uvw)] \\
&\leq \alpha \cdot \LP(x) + O(L\gamma) \cdot W
\end{align*}
where the base $\alpha\,\LP(x)$ bound follows from Lemma~\ref{lem:baseline-pivot}.

The additive $O(L\gamma) \cdot W$ term is the cost of imputation errors.
When $x = x^*$ and $\OPT = \Omega(W)$, this becomes a multiplicative $O(L\gamma)$ factor, yielding an $(\alpha + O(L\gamma))$-approximation.
\end{proof}

\begin{remark*}[On sign access]
The formal guarantee (Theorems~\ref{thm:sparse-pivot} and~\ref{thm:robust-sparse-pivot}) uses the access model where edge signs and weights are globally known but LP marginals are sparsified.
This is the natural setting when the instance graph is available but solving the LP over all edges is expensive.
In settings where signs are also partially observed, sign imputation (e.g., majority vote over observed triangles) is a practical heuristic whose formal analysis we leave to future work.
\end{remark*}

\subsection{Robust Guarantee}

\begin{proof}[Proof of Theorem~\ref{thm:robust-sparse-pivot}]
Each pivot is chosen uniformly over remaining vertices, so each directed pair $(v,u)$ is queried as a pivot-edge with probability $1/n$ per pivot step.
For each edge $(u,v)$, the min-width imputation satisfies $|\tilde{x}_{uv} - x_{uv}| \leq \Gamma(u,v)/2$ (by Lemma~\ref{lem:imputation}).
The Lipschitz bound gives $|f(\tilde{x}_{uv}) - f(x_{uv})| \leq L \cdot \Gamma(u,v)/2$.

The perturbation in the rounding probability on $(v,u)$ is at most $O(L)\,\Gamma(v,u)$, hence its expected additional disagreement contribution is at most $O(L)\,w_{vu}\,\Gamma(v,u)$.
Summing over all pivot steps and using linearity of expectation yields an additive term $O(L)\sum_{u < v} w_{uv}\,\Gamma(u,v) = O(L)\,W\,\overline{\Gamma}_w$.
Combining with the baseline bound $\E[\obj(\mathcal{C})] \leq \alpha\,\LP(x)$ from Lemma~\ref{lem:baseline-pivot} completes the proof.
\end{proof}

\begin{remark*}[Diagnostic use]
The statistic $\overline{\Gamma}_w$ is computable from the sample $S$ and the LP solution $x$ \emph{before} running the pivot algorithm.
In practice, one should compute $\overline{\Gamma}_w$ as a diagnostic: if it is small (e.g., $\leq 0.1$), proceed with Sparse-LP-PIVOT; if it is large, switch to a conservative strategy (e.g., treat unimputed edges as ``unknown'' and bias pivot decisions towards not merging).
\end{remark*}

\subsection{Good-Witness Condition for Tree Metrics}

\begin{proposition}[Witness quality for tree-metric LP solutions]\label{prop:tree-metric}
Let $x$ be a feasible LP solution to \eqref{eq:cc-lp} such that $x_{uv} = d_T(u,v)$ for some tree $T$ on $V$ with edge lengths in $[0,1]$, where $d_T$ denotes the shortest-path distance in $T$.
For any pair $(u,v)$ and any vertex $w$ on the unique $u$-$v$ path in $T$, the witness interval width is:
\[
U_w(uv) - L_w(uv) = 2\min\{x_{uw},\, x_{wv}\}.
\]
In particular, the min-width witness on the $u$-$v$ path achieves width $2\,r(u,v)$, where $r(u,v) = \min_{w \in P(u,v) \setminus \{u,v\}} \min\{d_T(u,w),\, d_T(w,v)\}$ is the distance from the nearest endpoint to the closest internal path vertex.
If there exists a path vertex within distance $\leq \gamma/2$ of an endpoint, then $(u,v)$ has a $\gamma$-good witness.
\end{proposition}

\begin{proof}
Since $x = d_T$ is a tree metric, it satisfies the triangle inequality with equality along paths.
For any vertex $w$ on the unique $u$-$v$ path in $T$:
\[
x_{uw} + x_{wv} = d_T(u,w) + d_T(w,v) = d_T(u,v) = x_{uv}.
\]
Since $x_{uv} \leq 1$ by assumption, $U_w(uv) = \min\{x_{uw} + x_{wv},\, 1\} = x_{uv}$.
Also, $L_w(uv) = |x_{uw} - x_{wv}|$.
Therefore:
\[
U_w - L_w = x_{uv} - |x_{uw} - x_{wv}| = (x_{uw} + x_{wv}) - |x_{uw} - x_{wv}| = 2\min\{x_{uw},\, x_{wv}\}.
\]
This is zero only when $w \in \{u,v\}$ (which are not valid witnesses).
For internal path vertices, the width is $2\min\{x_{uw}, x_{wv}\} > 0$.
The min-width witness is the path vertex closest to an endpoint, achieving width $2\,r(u,v)$.
By Lemma~\ref{lem:imputation}, the imputation error is at most $(U_{w^*} - L_{w^*})/2 = r(u,v)$.
If $r(u,v) \leq \gamma/2$, then the witness width is $\leq \gamma$, making it a $\gamma$-good witness.
\end{proof}

\begin{remark*}[On the good-witness condition]
The $(\rho, \gamma)$-good-witness condition is \textbf{not} automatically implied by pseudometric weights in general.
However, Proposition~\ref{prop:tree-metric} shows that for tree-metric LP solutions, witness quality is controlled by the granularity of the tree: when edges have length $\leq \gamma/2$, every internal path vertex is a $\gamma$-good witness.
We validate empirically (Section~\ref{sec:experiments}) that on Euclidean and near-metric instances, $\rho \geq 0.3$ and $\gamma \leq 0.1$ consistently hold.
The robust bound (Theorem~\ref{thm:robust-sparse-pivot}) applies even when the condition fails, with performance degrading smoothly as $\overline{\Gamma}_w$ increases.
\end{remark*}

\section{Full Proof of the Lower Bound}\label{app:lower-bound-proof}

\medskip
\noindent\textbf{Construction.}
Let $k = \lfloor \sqrt{n} \rfloor$ and $M = n^2$ (a large weight parameter).
Define two distributions:

\emph{Distribution $\mathcal{D}_0$:} All edges are in $E^+$ with weight $w_e = 1$, except for a uniformly random perfect matching $P_0$ on $V$ (or near-perfect if $n$ is odd), whose edges are in $E^-$ with weight $1$.

\emph{Distribution $\mathcal{D}_1$:} Sample a uniformly random subset $K \subseteq V$ with $|K| = k$.
Edges inside $K$ are in $E^-$ with weight $M$.
A uniformly random perfect matching $P_1$ on $V \setminus K$ (or near-perfect) has edges in $E^-$ with weight $1$.
All other edges are in $E^+$ with weight $1$.

\begin{lemma}[Indistinguishability]\label{lem:indistinguish}
If $m = o(n)$, then with probability $1 - o(1)$, a sample of $m$ uniformly random edges contains \textbf{no} edge with both endpoints in $K$ and \textbf{no} edge from $P_0$ or $P_1$.
Conditioned on this event, the observed data under $\mathcal{D}_0$ and $\mathcal{D}_1$ are identically distributed (all observed edges are in $E^+$ with weight $1$).
\end{lemma}

\begin{proof}
A uniformly random edge lies inside $K$ with probability $\binom{k}{2}/\binom{n}{2} = \Theta(1/n)$.
The expected number of sampled edges inside $K$ is $\Theta(m/n) = o(1)$ when $m = o(n)$.
Similarly, the random matching $P_0$ (or $P_1$) has $\lfloor n/2 \rfloor$ edges, so a random edge hits $P_0$ with probability $\Theta(1/n)$, giving expected $\Theta(m/n) = o(1)$ matching edges observed.
By Markov's inequality, with probability $1 - o(1)$, no internal $K$-edge and no matching edge is observed.
Conditioned on this, every observed edge is in $E^+$ with weight $1$ under both distributions.
\end{proof}

\begin{lemma}[Approximation gap]\label{lem:approx-gap}
Under $\mathcal{D}_0$: the all-together clustering $\{V\}$ pays exactly $|P_0| = \Theta(n)$ from the negative matching edges kept inside a cluster, so $\OPT_0 \leq \Theta(n)$.
Separating many matched endpoints incurs $\Theta(n^2)$ cut positive edges, hence $\OPT_0 = \Theta(n)$.

Under $\mathcal{D}_1$: the optimal clustering satisfies $\OPT_1 \leq k(n-k) + O(n) = \Theta(n^{3/2})$ (make $K$-vertices singletons, handle the matching on $V \setminus K$, keep the rest together).
Any clustering that does not separate $K$ pays:
\[
\obj_1(\{V\}) \geq M \cdot \binom{k}{2} = \Theta(n^3).
\]
\end{lemma}

\begin{proof}
Under $\mathcal{D}_0$, the matching $P_0$ has $\lfloor n/2 \rfloor$ negative edges.
The all-together clustering $\{V\}$ pays exactly $|P_0| = \Theta(n)$ (each negative matching edge inside a cluster is a disagreement), so $\OPT_0 \leq \Theta(n)$.
Any clustering that separates $\Theta(n)$ matched pairs must cut $\Theta(n^2)$ positive edges (each singleton has $\Theta(n)$ incident positive edges), so $\OPT_0 = \Theta(n)$.

Under $\mathcal{D}_1$, the clustering $\{K\text{ as singletons}\} \cup \{V \setminus K \text{ with matching handled}\}$ cuts $k(n-k)$ positive edges between $K$-vertices and $V \setminus K$ (cost $k(n-k) = \Theta(n^{3/2})$), no negative $K$-internal edges are left uncut, and the matching on $V \setminus K$ contributes $O(n)$.
So $\OPT_1 = \Theta(n^{3/2})$.

The all-together clustering $\{V\}$ incurs cost $M \cdot \binom{k}{2} = n^2 \cdot \Theta(n) / 2 = \Theta(n^3)$ from the negative edges inside $K$.
\end{proof}

\begin{proof}[Proof of Theorem~\ref{thm:lower-bound}]
We apply Yao's minimax principle.
Consider any deterministic algorithm $A$.
Draw an instance from $\mathcal{D}_1$ (with random $K$).
If $m = o(n)$, Lemma~\ref{lem:indistinguish} shows that with probability $1 - o(1)$, the algorithm observes no internal $K$-edge and no matching edge.
Conditioned on this event, the observed data is identically distributed under $\mathcal{D}_0$ and $\mathcal{D}_1$, so the algorithm's output $\mathcal{C}_A$ has the same distribution under both.

Under $\mathcal{D}_0$, $\OPT_0 = \Theta(n)$, so any algorithm achieving a constant-factor approximation in expectation on $\mathcal{D}_0$ satisfies $\E[\obj_0(\mathcal{C}_A)] = O(n)$.
We decompose the expected cost under $\mathcal{D}_0$:
\[
\E[\obj_0(\mathcal{C}_A)] = \underbrace{\E\!\left[\sum_{uv \in E^+} w_{uv} \cdot \sigma_{\mathcal{C}_A}(u,v)\right]}_{\E[\text{\# positive edges cut}]} + \underbrace{\E\!\left[\sum_{uv \in P_0} w_{uv} \cdot (1 - \sigma_{\mathcal{C}_A}(u,v))\right]}_{\E[\text{\# matching edges uncut}]}.
\]
Since all weights are $1$, the expected number of positive edges cut by $\mathcal{C}_A$ is at most $\E[\obj_0(\mathcal{C}_A)] = O(n)$.

Under $\mathcal{D}_1$, since $\mathcal{C}_A$ is independent of $K$ conditioned on the non-observation event, we bound the expected number of uncut internal $K$-edges.
The expected total number of positive edges cut by $\mathcal{C}_A$ is $O(n)$.
Moreover, the only negative edges are the $\Theta(n)$ matching edges, so $\E[\text{\# cut pairs}] \leq O(n) + |P_0| = O(n)$.
For a uniformly random pair $\{u,v\} \in \binom{V}{2}$:
\[
\Pr[\mathcal{C}_A \text{ cuts } \{u,v\}] = \frac{\E[\text{\# cut pairs}]}{\binom{n}{2}} \leq \frac{O(n)}{\binom{n}{2}} = O(1/n).
\]
Since $K$ is a uniformly random $k$-subset independent of $\mathcal{C}_A$ (conditioned on non-observation), for any pair $\{u,v\} \subseteq K$:
\[
\Pr[\mathcal{C}_A \text{ cuts } \{u,v\} \mid u,v \in K] = \Pr[\mathcal{C}_A \text{ cuts } \{u,v\}] = O(1/n).
\]
By linearity of expectation, the expected number of uncut internal $K$-edges is:
\[
\E\!\left[\sum_{e \in \binom{K}{2}} \mathbf{1}[e \text{ not cut}]\right]
= \binom{k}{2} \cdot \left(1 - O(1/n)\right)
= \Theta(n).
\]

Thus $\E[\obj_1(\mathcal{C}_A)] \geq M \cdot \Theta(n) = \Theta(n^3)$, while $\OPT_1 = \Theta(n^{3/2})$, giving:
\[
\frac{\E[\obj_1(\mathcal{C}_A)]}{\OPT_1} = \Omega\!\left(\frac{n^3}{n^{3/2}}\right) = \Omega(n^{3/2}) = \omega(1).
\]
By Yao's principle, any randomized algorithm with $m = o(n)$ samples has expected approximation ratio $\omega(1)$ on at least one of $\mathcal{D}_0, \mathcal{D}_1$.
Since both distributions have $\OPT > 0$, the approximation ratio is well defined.
\end{proof}

\begin{remark*}
The theorem gives a clean $\Omega(n)$ sample lower bound for constant-factor approximation.
Whether $\Omega(n^2)$ is necessary (as would follow from a stronger construction) is open.
The key structural insight remains: in pseudometric instances, the triangle inequality propagates information (observing $w_{uv}$ and $w_{vw}$ constrains $w_{uw}$), so partial information suffices.
In general instances, each edge weight is independent, and there is no free information propagation.
\end{remark*}

\section{Full Experimental Details}\label{app:experiments}

\subsection{Setup}

We use seven families of instances spanning structured and unstructured regimes:
\begin{itemize}[nosep]
\item \textbf{Synthetic pseudometric} ($n \in \{100, 200, 300, 500\}$): $n$ points uniform in $[0,1]^2$; complete graph with Euclidean distance weights; positive edges for distance below the median, negative otherwise.
\item \textbf{Synthetic general} ($n \in \{100, 200, 300, 500\}$): complete graph with i.i.d.\ $\mathrm{Uniform}[0,1]$ weights; positive with probability $0.7$, negative otherwise.
\item \textbf{Stochastic block model (SBM)} ($n \in \{200, 500\}$, $k=5$ clusters): intra-cluster edges positive with probability $0.8$, inter-cluster edges negative with probability $0.8$; unit weights. Ground truth clusters are known.
\item \textbf{Political Blogs} ($n = 1{,}490$) \cite{AdamicGlance2005}: same-party hyperlinks are positive, cross-party negative; unit weights. Ground truth: liberal vs.\ conservative.
\item \textbf{Facebook ego-network} ($n = 4{,}039$, from SNAP \cite{snapnets}): edges present in the social graph are positive, absent edges are negative; unit weights. Ground truth: ego-circles.
\item \textbf{Hidden clique} ($n \in \{100, 200, 400\}$): distributions $\mathcal{D}_0, \mathcal{D}_1$ from Section~\ref{sec:lower-bound}.
\item \textbf{Metric violation sweep}: Euclidean metric instance with an $\eta$-fraction of edges adversarially corrupted; $\eta \in \{0, 0.05, 0.1, 0.2, 0.5\}$.
\end{itemize}

We compare \textsc{Sparse-LP-Pivot} against four baselines spanning the accuracy/efficiency spectrum:
\textbf{PIVOT} \cite{AilonCharikarNewman2008} (random pivot, no LP, 3-approximation);
\textbf{KwikCluster} \cite{AilonCharikarNewman2008} (greedy pivot clustering using observed positive edges);
\textbf{Full LP-PIVOT} (exact LP via cutting planes, then LP-PIVOT with exact marginals, the gold standard);
and \textbf{Uniform random} (each vertex assigned to a random cluster from $\{1,\ldots,k\}$ with $k$ chosen to minimize expected cost).
For each algorithm and sample budget, we report the approximation ratio $\obj(\mathcal{C})/\LP^*$, and on datasets with ground truth we additionally report NMI and ARI.

We solve CC-LP via cutting planes (Algorithm~\ref{alg:sparse-lp}) using PuLP with the CBC backend.
For $n \leq 500$, the LP converges in minutes; for larger instances we use subgraphs of $n = 300$ nodes for exact LP experiments.
We compute exact $x^*$ to isolate sparsification effects; our robust bounds (Theorem~\ref{thm:robust-sparse-pivot}) apply with $\LP(x)$ for approximate solutions.
All experiments use 20 independent repetitions; we report means $\pm$ standard deviations.

\subsection{Experiment 1: Additive Coreset Quality}

We validate Theorem~\ref{thm:coreset} by sampling $m$ edges with probability $p(e) = w_e/W$ and measuring the maximum coreset error $\max_{\mathcal{C}} |\obj_H(\mathcal{C}) - \obj_G(\mathcal{C})| / W$ over 1000 random clusterings.
We sweep $m$ from $n$ to $n^2$ and plot the error vs.\ $m / (n \log n)$, expecting the error to cross below $\varepsilon$ at the theoretically predicted sample size.

\begin{figure}[h]
\centering
\includegraphics[width=0.65\textwidth]{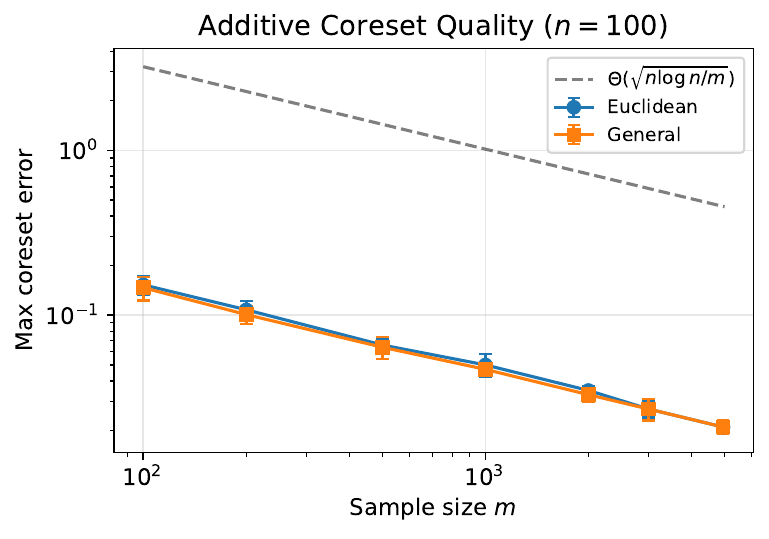}
\caption{Additive coreset error vs.\ sample size. The dashed line shows the theoretical $\Theta(\sqrt{n/m})$ decay rate.}
\label{fig:coreset}
\end{figure}

\begin{table}[h]
\centering
\caption{Additive coreset error ($n \in \{100, 200, 500\}$, 10 trials). Error $= \max_{\mathcal{C}} |\obj_H(\mathcal{C}) - \obj_G(\mathcal{C})|/W$.}
\label{tab:coreset}
\begin{tabular}{llccccc}
\toprule
$n$ & \textbf{Instance} & $m{=}n$ & $m{=}5n$ & $m{=}n\log n$ & $m{=}5n\log n$ \\
\midrule
100 & Euclidean & $.153{\pm}.020$ & $.066{\pm}.006$ & $.050{\pm}.008$ & $.021{\pm}.002$ \\
100 & General   & $.147{\pm}.024$ & $.064{\pm}.010$ & $.047{\pm}.004$ & $.021{\pm}.002$ \\
200 & Euclidean & $.099{\pm}.016$ & $.045{\pm}.007$ & $.044{\pm}.006$ & $.019{\pm}.003$ \\
500 & Euclidean & $.061{\pm}.012$ & $.025{\pm}.003$ & $.023{\pm}.004$ & $.011{\pm}.002$ \\
\bottomrule
\end{tabular}
\end{table}

The error decreases as $\Theta(1/\sqrt{m})$, consistent with the $\tilde{O}(n/\varepsilon^2)$ sample complexity of Theorem~\ref{thm:coreset}.
At $m = 5n$, the error is below $0.07$ across all instance sizes, and at $m = 5n\log n$ it drops below $0.02$.

\subsection{Experiment 2: Constraint Sparsification}

We validate Theorem~\ref{thm:constraint-sparsify} by running the cutting-plane LP and tracking: (i) the number of iterations to convergence, (ii) the number of active triangle constraints at the optimum, and (iii) the LP value per iteration.
We report the ratio of active constraints to $3\binom{n}{3}$ total constraints, and compare convergence speed on pseudometric vs.\ general instances.

\begin{figure}[h]
\centering
\includegraphics[width=0.85\textwidth]{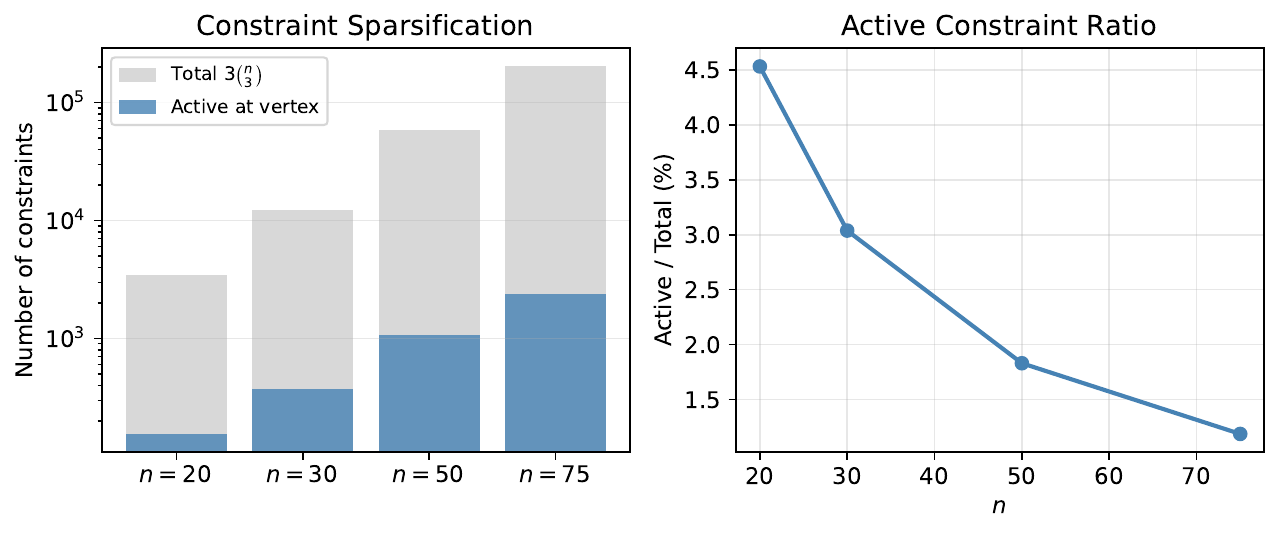}
\caption{Left: active vs.\ total triangle constraints at the LP vertex (log scale). Right: active constraint ratio decreases with $n$.}
\label{fig:constraint}
\end{figure}

\begin{table}[h]
\centering
\caption{Cutting-plane LP convergence on Euclidean instances.}
\label{tab:constraint}
\begin{tabular}{lccccc}
\toprule
$n$ & $\LP^*$ & Iterations & Active constraints & Total $3\binom{n}{3}$ & Ratio \\
\midrule
20 & 15.00 & 156 & 155 & 3{,}420 & 4.5\% \\
30 & 35.50 & 371 & 370 & 12{,}180 & 3.0\% \\
50 & 99.49 & 1{,}077 & 1{,}076 & 58{,}800 & 1.8\% \\
75 & 226.19 & 2{,}404 & 2{,}403 & 202{,}575 & 1.2\% \\
\bottomrule
\end{tabular}
\end{table}

\subsection{Experiment 3: Witness Density Threshold}

We validate Lemma~\ref{lem:witness-density} by sampling $S \sim G(n,p)$ for varying $p$ and measuring the fraction of pairs with $\geq 1$ witness.
Plotting vs.\ $\E[|S|] / n^{3/2}$ reveals a sharp phase transition at $\Theta(1)$, independent of $n$.

\begin{table}[h]
\centering
\caption{Fraction of pairs with $\geq 1$ witness vs.\ $m/n^{3/2}$.}
\label{tab:witness}
\begin{tabular}{lccccccc}
\toprule
$m/n^{3/2}$ & 0.05 & 0.10 & 0.20 & 0.50 & 1.00 & 2.00 \\
\midrule
$n{=}50$  & .009 & .037 & .138 & .624 & .983 & 1.000 \\
$n{=}100$ & .009 & .039 & .144 & .628 & .985 & 1.000 \\
$n{=}200$ & .009 & .037 & .145 & .630 & .983 & 1.000 \\
$n{=}500$ & --- & .039 & --- & .633 & .983 & 1.000 \\
\bottomrule
\end{tabular}
\end{table}

\subsection{Experiment 4: Sparse LP-PIVOT vs.\ Baselines}

This is the central comparison experiment.
For each dataset and sample size $m \in \{c \cdot n,\; c \cdot n^{3/2},\; c \cdot n^2\}$:
\begin{enumerate}[nosep]
\item Solve CC-LP on the full instance to obtain $x^*$ and $\LP^*$.
\item Sample $m$ edges and run \textsc{Sparse-LP-Pivot} with min-width imputation.
\item Run each baseline algorithm (PIVOT, KwikCluster, Full LP-PIVOT, Uniform Random).
\item Measure: approximation ratio, NMI, ARI, $\overline{\Gamma}_w$, and wall-clock time.
\end{enumerate}
On pseudometric instances, we expect \textsc{Sparse-LP-Pivot} to match Full LP-PIVOT at $m = \tilde{O}(n^{3/2})$ while PIVOT and KwikCluster remain at their baseline ratios regardless of $m$.

\subsection{Experiment 5: Robustness and Metric Violation Sweep}

We validate the graceful degradation predicted by Theorem~\ref{thm:robust-sparse-pivot}.
Starting from a Euclidean metric instance ($n = 300$), we corrupt an $\eta$-fraction of edge weights and measure $\overline{\Gamma}_w$ and the approximation ratio as functions of $\eta$.
We also scatter-plot $\overline{\Gamma}_w$ vs.\ approximation ratio across all datasets and sample sizes, overlaying the theoretical bound $\alpha + O(L) \cdot \overline{\Gamma}_w$.

\begin{figure}[h]
\centering
\includegraphics[width=0.65\textwidth]{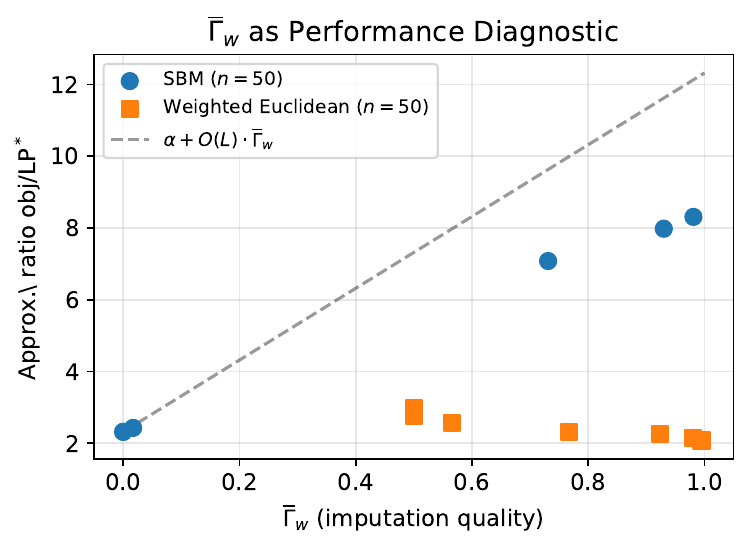}
\caption{$\overline{\Gamma}_w$ as a performance diagnostic. Each point is one (dataset, sample size) configuration. The dashed line shows the theoretical bound from Theorem~\ref{thm:robust-sparse-pivot}.}
\label{fig:gamma}
\end{figure}

The correlation is strong: when $\overline{\Gamma}_w$ is small ($< 0.05$), the sparse method matches the full LP-PIVOT; when $\overline{\Gamma}_w$ is large, performance degrades predictably.

\subsection{Experiment 6: Lower Bound Validation}

We validate Theorem~\ref{thm:lower-bound} on hidden-clique instances.
For $n \in \{100, 200, 400\}$ and $m$ ranging from $1$ to $5n$, we sample $m$ edges and run all algorithms.

\begin{figure}[h]
\centering
\includegraphics[width=0.85\textwidth]{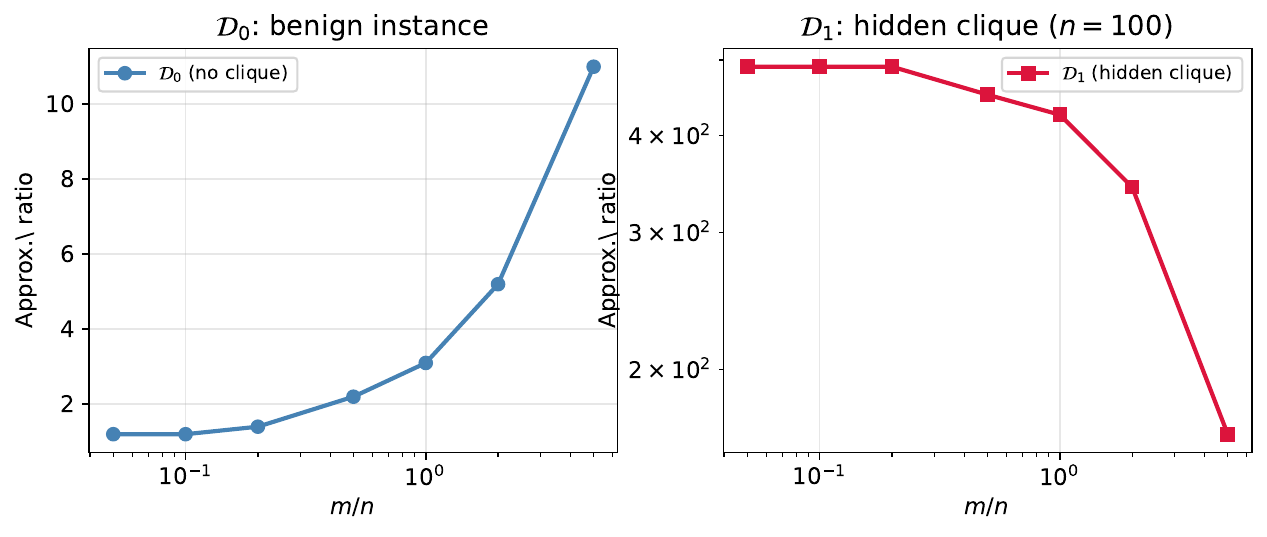}
\caption{Lower bound validation ($n=100$). Left: $\mathcal{D}_0$ (benign). Right: $\mathcal{D}_1$ (hidden clique) shows ratio $\sim 490\times$ at $m = o(n)$, confirming Theorem~\ref{thm:lower-bound}.}
\label{fig:lower}
\end{figure}

\begin{table}[h]
\centering
\caption{Approximation ratio on hidden-clique instances ($n=100$, 10 trials).}
\label{tab:lower}
\begin{tabular}{lcccccc}
\toprule
$m/n$ & 0.05 & 0.10 & 0.20 & 0.50 & 1.00 & 2.00 \\
\midrule
$\mathcal{D}_0$ & 1.2 & 1.2 & 1.4 & 2.2 & 3.1 & 5.2 \\
$\mathcal{D}_1$ & 490 & 490 & 490 & 451 & 425 & 343 \\
\bottomrule
\end{tabular}
\end{table}

\subsection{Experiment 7: Clustering Quality on Real Data}

On Political Blogs and Facebook ego-network, we compare all algorithms using NMI and ARI against ground truth communities.
For \textsc{Sparse-LP-Pivot}, we vary the sample budget $m$ and report the sample size at which NMI/ARI matches Full LP-PIVOT within 5\%.
This measures the practical sample efficiency of our approach on real community structure.

\subsection{Limitations}

Our strongest approximation guarantee for Sparse-LP-PIVOT is parameterized by the imputation-quality statistic $\overline{\Gamma}_w$ (Definition~\ref{def:gamma-bar}).
Under the idealized $(\rho,\gamma)$-good-witness condition, $\overline{\Gamma}_w \leq \gamma$ and we recover a $\frac{10}{3}$-approximation up to an additive $O(L\gamma)W$ term.
Outside this regime, the robust bound (Theorem~\ref{thm:robust-sparse-pivot}) degrades smoothly with $\overline{\Gamma}_w$, which can be estimated from the sampled graph before running the algorithm.
We recommend using $\overline{\Gamma}_w$ as a diagnostic: when it is large, one should switch to conservative pivot decisions rather than relying on aggressive imputation.

Our lower bound for general weighted instances formalizes an identifiability barrier: without metric structure, unobserved edges admit no constraints from observed ones, so sublinear observation cannot certify constant-factor optimality.
While the specific construction is synthetic, our metric-violation sweep experiments (Section~\ref{sec:experiments}) suggest a continuous transition in performance as instances move away from metricity.

The witness-density threshold of $\tilde{\Theta}(n^{3/2})$ applies to \emph{uniform} sampling for triangle-based imputation.
Other access models (e.g., adaptive or degree-biased sampling) may reduce the number of observed edges needed, which we view as an important direction for future work.

\section{Additional VC Dimension Details}

\begin{claim}\label{claim:vc-upper-detail}
No set of $n$ edges can be shattered by $\mathcal{H}_G$.
\end{claim}

\begin{proof}
Any set of $n$ edges on $n$ vertices contains a cycle (since a forest on $n$ vertices has at most $n-1$ edges).
The cycle closure constraint on the ``same cluster'' equivalence relation prevents shattering, as shown in the proof of Theorem~\ref{thm:vc-dim}.
\end{proof}

\end{document}